%% file: main.tex
\documentclass[10pt]{article}
\usepackage[utf8]{inputenc}
\input{headers}
\title{Differentially-Private Sublinear-Time Clustering}

\author{Jeremiah Blocki, Elena Grigorescu, Tamalika Mukherjee\\ 
Department of Computer Science, Purdue University.\\ \{jblocki, elena-g, tmukherj\}@purdue.edu
\thanks{ J. B. was supported in part by NSF CNS-1931443 and NSF CCF-1910659. E.G  was supported in part by NSF CCF-1910659 and NSF CCF-1910411.}
 }

\begin{document}

\maketitle

\begin{abstract}
    Clustering is an essential primitive in unsupervised machine learning.  We bring forth the problem of sublinear-time differentially-private clustering as a natural and well-motivated direction of research. 
    We combine the $k$-means and $k$-median sublinear-time results of Mishra et al. (SODA,  2001) and of Czumaj and Sohler (Rand. Struct. and Algorithms, 2007) with recent results on private clustering of Balcan et al. (ICML 2017), Gupta et al. (SODA, 2010) and Ghazi et al. (NeurIPS, 2020) to obtain sublinear-time private $k$-means and $k$-median algorithms via subsampling.
     We also investigate the privacy benefits of subsampling for group privacy. 
 
\end{abstract}

\section{Introduction}

Preserving privacy in data collection and distribution have long been a concern for industrial and governmental agencies, who are now rapidly adopting privacy standards and policies~\cite{kifer2020guidelines,appledp,prochlo,erlingsson2014rappor}. Differential privacy~\cite{DMNS06} is the gold standard of privacy protection. A randomized function computed on a database is {\em differentially private} if the distribution of the function's output does not change by much with the presence or absence of an individual record. While existing research mostly focuses on computing efficient polynomial-time differentially-private algorithms, in dealing with a large amount of data, even linear-time algorithms may be prohibitive in costs. Hence, algorithms that can quickly output approximately accurate solutions while preserving privacy are of great interest in real-world computations on large datasets (e.g., billions of Facebook or Google, or Microsoft users). 

\begin{definition}
A randomized algorithm $\cM$ taking as input a dataset $D$ is $(\epsilon,\delta)$-{\em differentially private} if for any two neighboring\footnote{Datasets $D$ and $D'$ are \emph{neighboring} if removing or adding one point in $D$ results in $D'$; alternatively, if changing one data point in $D$ results in $D'$.} data sets $D$ and $D'$, and for any subset $C$ of outputs of $\cM$ it holds that $\Pr[\cM(D) \in C] \leq e^\epsilon \cdot \Pr[\cM(D') \in C] +\delta$. If $\delta=0$, $\cM$ is $\epsilon$-{\em differentially private. }
\end{definition}

However, despite the fact that the literature on differentially private algorithms has grown rapidly in recent years, sublinear-time private algorithms for many natural problems are still lacking. In this work we focus on clustering problems and provide some basic sublinear-time private solutions derived from the existing efficient non-private analogues.
 

Clustering is an essential primitive in unsupervised machine learning. Since many machine learning models deal with sensitive data, private clustering has been studied extensively in the polynomial-time setting~\cite{NRS07,feldman2009private, GLMRT2010, su2016differentially, balcanDLMZ17a,KS18, huang2018optimal,shechner2020private, Ghazi2020DifferentiallyPC, stemmer2020locally, nissim2018clustering}. Two of the most widely studied variants of clustering are the $k$-median and $k$-means problem. In the $k$-median problem, we are given $n$ data points, and the goal is to find $k$ centers that minimize the sum of distances from the data points to their nearest centers. The setup is the same for $k$-means, except the goal is to find $k$ centers that minimize the sum of the squares of distances from the data points to their nearest centers. Both types of clustering are classical problems, and there is a rich field of research devoted to them in the non-private setting~\cite{AGKMMP01, charikar2002constant, har2004coresets, chen2006k, chen2008constant, awasthi2010stability, li2016approximating, ahmadian2019better, ostrovsky2013effectiveness}.

\subsection{Contributions}
 We bring forth the problem of sublinear-time private clustering as a natural and well-motivated direction of research, and show some basic results derived from the non-private analogues on subsampled data.
  We expect that our results will entice further interest in understanding the best privacy guarantees in sublinear clustering settings.

{\bfseries Private sublinear clustering.}
We combine the techniques of sublinear-time clustering algorithms from Mishra \etal~\cite{MOP11} and Czumaj \etal~\cite{Czumaj2007SublineartimeAA} with the private polynomial-time approximation clustering algorithms with a constant multiplicative factor of Balcan~\etal~\cite{balcanDLMZ17a}, Gupta et al.~\cite{GLMRT2010} and Ghazi et al.~\cite{ Ghazi2020DifferentiallyPC} to obtain private sublinear-time clustering algorithms for $k$-median and $k$-means clustering in metric spaces, as well as better approximation guarantees for the particular case of Euclidean space. 
To the best of our knowledge, these are the first \emph{sublinear-time differentially-private clustering} algorithms formalized in the privacy literature. 

Let $(V,d)$ be an arbitrary metric space. Given an input set $D \subseteq V$, the goal of the \emph{$k$-median} clustering problem is to find a set of centers (\ie a clustering) $\{c_1,\ldots,c_k\} \subseteq V$ such that the cost of clustering $\sum_{x \in D} \min_i d(x,c_i)$ is minimized. The goal of the \emph{$k$-means} clustering problem is to find a clustering $\{c_1,\ldots,c_k\} \subseteq V$ such that the cost $\sum_{x \in D} \min_i d^2(x,c_i)$ is minimized. An $(\alpha,\gamma)$-approximation algorithm for $k$-median (equivalently for $k$-means) takes as input a set $D$ (say), and outputs clustering $\hat{C}:=\{\hat{c_1},\ldots,\hat{c_k}\}$ such that $\sum_{x \in D} \min_i d(x,\hat{c}_i) \leq \alpha \cdot \sum_{x \in D} \min_i d(x,c_i)+ \gamma$, where ${c_1},\ldots,{c_k}$ is the optimum clustering for $D$.

We analyze the following sampling algorithm: pick a random sample from the input set; run a private $k$-median (or $k$-means) polynomial-time approximation algorithm on the random sample to obtain a $k$-median (or $k$-means) clustering of the sample; output this clustering. 
We show that for a small sample size, the average cost of the clustering induced by the random sample is not too far from the average cost of the optimum clustering of the input set. Our analysis closely follows the works of Mishra \etal~\cite{MOP11} and Czumaj~\etal~\cite{Czumaj2007SublineartimeAA}, who gave sublinear time algorithms for clustering in the non-private setting using a constant $\alpha$-approximation polynomial-time algorithm as a black-box.  
We extend their analysis to handle the case of using an $(\alpha,\gamma)$-approximation polynomial time algorithm as a black-box
\footnote{We note that an additive approximation factor $\gamma>0$ is unavoidable for any \emph{private} clustering algorithm, thus this extension was necessary. To see why, consider the following two multisets of input data points $D_1=\{x_1, \ldots, x_1, x_2, \ldots x_{k-1}, x_k\}$ and $D_2=\{x_1, \ldots, x_1, x_2, \ldots x_{k-1}, x_{k+1}\}$, where $x_1$ occurs in both sets $n-k+1$ times. Note that the optimal cost for $k$-median in both cases is zero, and in the non-private setting, the algorithm can simply output $\{x_1,\ldots,x_k\}$ for $D_1$ or $\{x_1,\ldots,x_{k+1}\}$ for $D_2$ as the solution. But a private algorithm must have an additive error since the set of centers computed by our algorithm cannot be affected by the change of replacing $x_k$ in $D_1$ by $x_{k+1}$ in $D_2$, and the input points being private, should not be revealed by our algorithm.}. 
The approximation guarantee achieved by our algorithm is essentially the same as that of the black-box private algorithms (modulo an extra additive factor of $\eta$)
\footnote{This extra additive factor of $\eta$ is unavoidable in order to design clustering algorithms with running time $o(n)$, see~\cite{Czumaj2007SublineartimeAA} for an exposition.}
For an arbitrary metric space $(V,d)$ consisting of $n$ points and input set $D \subseteq V$,
\begin{enumerate}
    \item Assuming a private $(\alpha,\gamma)$-factor approximation $k$-median algorithm, that runs in time $T(n)$, we can draw a sample $S \subseteq D$ of size $\text{poly}\left(\alpha, k \ln(n)\right)$ and obtain a $k$-median clustering $\hat{c}_S$ in time $T(s)$ such that with high probability ${\sf avg\text{-}cost}(\hat{c}_S) \leq \alpha \cdot {\sf avg\text{-}cost}({c}_D) + \gamma + \eta$, where $c_D$ is the optimum $k$-median clustering of $D$. 
    \item Assuming a private $(\alpha,\gamma)$-factor approximation $k$-means algorithm that runs in time $T(n)$, we can draw a sample $S \subseteq D$ of size $\text{poly}\left(\alpha, k \ln(n)\right)$ and obtain a $k$-means clustering $\hat{c}_S$ such that with high probability ${\sf avg\text{-}cost}(\hat{c}_S) \leq \alpha \cdot {\sf avg\text{-}cost}({c}_D) + \gamma + \eta$, where $c_D$ is the optimum $k$-means clustering of $D$. 
\end{enumerate}
For the special case of $k$-median in $d$-dimensional Euclidean space, we achieve a sample complexity that is independent of the size of the input set $D \subseteq \bbR^d$ consisting of $n$ points. 
\begin{enumerate}
\item Assuming a private $(\alpha,\gamma)$-factor approximation $k$-median algorithm, that runs in time $T(n)$, we can draw a sample $S \subseteq D$ of size $\text{poly}(\alpha, dk \ln(n))$ and obtain a $k$-median clustering $\hat{c}_S$ in time $T(s)$ such that with high probability ${\sf avg\text{-}cost}(\hat{c}_S) \leq \alpha \cdot {\sf avg\text{-}cost}(c_D) + \gamma + \eta$, where $c_D$ is the optimum $k$-median clustering of $D$.
\item Assuming a private $(\alpha,\gamma)$-factor approximation $k$-means algorithm, that runs in time $T(n)$, we can draw a sample $S \subseteq D$ of size $\text{poly}(\alpha, dk \ln (n))$ and obtain a $k$-means clustering $\hat{c}_S$ in time $T(s)$ such that with high probability ${\sf avg\text{-}cost}(\hat{c}_S) \leq \alpha \cdot {\sf avg\text{-}cost}(c_D) + \gamma + \eta$, where $c_D$ is the optimum $k$-means clustering of $D$. 
\end{enumerate}

{\bfseries Group privacy for sampling algorithms.}
Group privacy ensures that for pairs of inputs that differ on a small number of points, the privacy loss is still bounded. For example, in the setting of a health survey administered to families, a family may wish to preserve all its members' privacy. Any $\epsilon$-differentially private algorithm, ensures $(g \epsilon,0)$-privacy for groups of size $g$. We show that our random sampling algorithm has better group privacy guarantees.
In other words, an algorithm that runs an $(\eps,0)$-differentially private mechanism on a subsample is $(T \cdot \eps, \delta_T)$-differentially private for groups of size $g$, for $0 \leq T \leq g$, where $\delta_T$ is the probability of the number of samples from the $g$ elements is $>T$. We note that $\delta_T$ is often negligible even for $T\ll g$. In such cases, the guarantee of $(T \cdot \eps, \delta_T)$-differential privacy is arguably much stronger than the naive guarantee of $(g \eps,0)$-group privacy.

\subsection{Related Work} \label{subsec:rel-work}
Sublinear-time approximate $k$-median clustering of a space in which the diameter of points is bounded was introduced by Mishra \etal~\cite{MOP11}. They modeled clusterings as functions and studied the quality of $k$-median clusterings obtained by random sampling using computational learning theory techniques. For a metric space, their work shows that if we sample a set of size $\text{poly}(\alpha,k \ln n)$ and run an $\alpha$-approximation clustering algorithm on the sample, then with high probability, the set of centers outputted is at most $2 \alpha\cdot {\sf avg\text{-}cost}(c_{OPT})+ \eta$. Their sampling model was adapted by Czumaj \etal ~\cite{Czumaj2007SublineartimeAA}, who achieved a sample complexity that is independent of $n$ for the $k$-median clustering problem in arbitrary metric spaces. They also extended the random sampling model and their analysis to give sublinear-time results for clustering variants such as $k$-means and min-sum clustering. 

Private clustering was first studied by Gupta \etal~\cite{GLMRT2010}, and Feldman \etal~\cite{feldman2009private}. 
Gupta \etal~\cite{GLMRT2010} modified the local search algorithm for $k$-median by Arya \etal~\cite{AGKMMP01} to choose candidate centers in each iteration via the exponential mechanism~\cite{mcsherry2007mechanism} and produced a polynomial-time algorithm that achieves $(O(1),\Tilde{O}(k^2M))$-approximation ($M$ is the diameter of the space) in discrete spaces. However, their algorithm is highly inefficient in Euclidean space (see~\cite{KS18} for a detailed exposition).
A recent line of work has focused on producing an efficient polynomial time algorithm for clustering that achieves a constant (multiplicative) factor approximation in high-dimensional Euclidean space by adopting the techniques of Gupta \etal while maintaining efficiency~\cite{balcanDLMZ17a,KS18}. A different approach to private clustering was taken by~\cite{feldman2009private}. They gave an efficient algorithm for $k$-median and $k$-means in Euclidean space by introducing the notion of private coresets. A recent line of work has adopted their techniques to give clustering algorithms with better approximation guarantees and efficiency~\cite{feldman2017coresets, nissim2018clustering, Ghazi2020DifferentiallyPC}.

Privacy amplification by subsampling has been formally studied by Balle \etal~\cite{balle2018privacy}. Our result is a simple observation that tailors the privacy amplification achieved with respect to group privacy for a generic sampling algorithm that runs a private algorithm as a black-box in the sampling step.

\section{Private Sublinear time Approximate Clustering}\label{sec:priv-sla-cluster}
In this section we describe the generic random sampling algorithm $\cA'$ using a private $(\epsilon,\delta)$-differentially private as a black-box, and in the sequel, we show that $\cA'$ is $(\epsilon',\delta')$-differentially private where $\epsilon'$ and $\delta'$ are functions of $\epsilon,\delta$ (see \theoremref{privacy-approx-k-med}). 
Additionally, we give the accuracy of $\cA'$, \ie, the minimum sample size needed to guarantee that with high probability the approximate clustering cost of the sample $S$ will be close to the true clustering cost of the input set $D$ when $D$ is a subset of an arbitrary metric space (see \theoremref{metric-acc-k-med-sub}, \theoremref{metric-acc-k-mean-sub}) and in the special case of Euclidean space (see \theoremref{euclid-acc-k-med-sub}, \theoremref{euclid-acc-k-mean-sub}). 

\begin{remark}
 For the metric setting, both \cite{Czumaj2007SublineartimeAA} and \cite{MOP11} consider clusterings where the centers are a subset of the set of input data points (this type of clustering is known as discrete clustering). By carefully conditioning on this requirement, \cite{Czumaj2007SublineartimeAA} can  make the sample complexity independent of $n$. Unfortunately, due to privacy concerns, we must consider the set of chosen $k$ centers to be any subset of the entire metric space, and not restricted to the input set (this type of clustering is known as continuous clustering). Thus we cannot hope to achieve a sample complexity independent of $n$ in the metric setting, using their approach. 
 
 We present techniques used by \cite{MOP11} for our $k$-median clustering analysis and describe the techniques used by \cite{Czumaj2007SublineartimeAA} for our $k$-means clustering analysis. 
\end{remark}

 
\subsection{Generic Algorithm $\cA'$}
We first present the basic sampling algorithm we employ, this model was first introduced in \cite{MOP11}. Note that the sampling probability $\xi$ should be chosen as $o(1)$. 
\begin{algorithm}[h]
\caption{General Sampling Scheme $\cA'$}
\begin{algorithmic}
\STATE On input dataset $D$, and sampling probability parameter $\xi$
\STATE Sample each element of $D$ independently w.p. $\xi$ and let $S$ be the sample set.
\STATE Run $(\eps,\delta)$-DP $(\alpha,\gamma)$-approximation algorithm $\cA$ on $S$ to compute a set of private $k$-centers for $S$, denoted by $C^*$.
\STATE Output the clustering $C^*$.
\end{algorithmic}
\end{algorithm}




\subsection{Privacy of $\cA'$} \label{subsec:privacy-proof}
In this section we show that for an algorithm $\cA'(D)$ which takes $D$ as input and runs a $(\eps,\delta)$-differentially private algorithm $\cA$ on random sample $S \subseteq D$, it is the case that  $\cA'$ is $(\eps', \delta')$-differentially private. Many works prove something similar to the following, e.g.,~\cite{KLNRS08,li2012sampling, balle2018privacy}. We include the proof here for the sake of clarity and completeness~\footnote{Our proof slightly generalizes the analysis given by Adam Smith in his blog post~\cite{smith_2009}.}.
\begin{theorem}\label{thm:privacy-approx-k-med}
If $\cA$ is an $(\eps,\delta)$-differentially private algorithm, and algorithm $\cA'$ is the generic sampling algorithm defined above where each element is sampled independently with probability $\xi$, then $\cA'$ is $(\eps',\delta')$-differentially private, where $\eps ' = \ln \max\left\{ \xi(e^{\eps}-1)+1, (\xi(e^{-\eps}-1)+1)^{-1} \right\}$, and $\delta' = \max\{\frac{e^{-\eps} \delta \xi}{(\xi(e^{-\eps}-1)+1)}, \delta \xi \}$. \end{theorem}
Observe that if $\cA$ is $(\eps,\delta)$-DP, then trivially, $\cA'$ is also $(\eps,\delta)$-DP. The privacy bounds achieved in the above theorem are significantly better than these naive bounds. For example, if we consider $\eps=0.5,\ \xi=0.001$, for any $\delta \in [0,1)$, we achieve $\eps'<0.00065$, and $\delta'=0.001 \delta$, which is orders of magnitude smaller than $\eps$ and $\delta$.  

\begin{proof}
Let $D$ and $D'$ be neighboring data sets \ie $D'= D \cup \{x\}$, and let us fix any subset $C$ of all possible outcomes in the output space. 

 Let $S$ be the set sampled from $D$ and $S'$ be the set sampled from $D'$, where each element from $D$ is independently chosen to belong to $S$ w.p. $\xi$, and similarly, each element from $D'$ is selected in $S'$ w.p. $\xi$.
 Since $\cA$ is differentially private we have that for any valid subset of outcomes $C$, for all $y\nin S$,
 \begin{align}
\Pr[\cA(S\cup \{y\}) \in C]&\leq e^\eps \Pr[\cA(S) \in C]+\delta\;, \label{eq:dp-1}\\
\Pr[\cA(S) \in C]&\leq e^\eps \Pr[\cA(S\cup \{y\}) \in C]+\delta\;. \label{eq:dp-2}
 \end{align}

We will show that
$\Pr_{S'\leftarrow D'} [\cA(S')\in C] \leq e^{\eps'} \Pr_{S\leftarrow D} [\cA(S) \in C]+\delta',$ and $\Pr_{S\leftarrow D} [\cA(S)\in C] \leq e^{\eps'} \Pr_{S'\leftarrow D'} [\cA(S') \in C]+\delta',$ which shows that $\Pr[\cA'(D') \in C] \leq e^{\eps'}\Pr[\cA(D) \in C] + \delta',$ and $\Pr[\cA'(D) \in C] \leq e^{\eps'}\Pr[\cA'(D') \in C] + \delta',$ and hence $\cA'$ is differentially private.

Indeed, using eq.~(\ref{eq:dp-1}), we have
\begin{eqnarray*}
\Pr_{S'\leftarrow D'}[\cA(S')\in C]&=& \Pr[x\nin S'] \Pr[\cA(S')\in C ~| x\nin S']+\Pr[x\in S'] \Pr[\cA(S')\in C ~| x \in S']\\
&=& \left(1-\xi \right) \cdot \Pr_{S\leftarrow D}[\cA(S)\in C] + \xi \cdot \Pr[\cA(S')\in C ~| x \in S'] \\
&\leq& \left(1-\xi\right)\cdot \Pr_{S\leftarrow D}[\cA(S)\in C]+\xi\cdot  (e^{\eps} \Pr_{S\leftarrow D} [\cA(S)\in C] +\delta)\\
&=& (\xi(e^{\eps}-1)+1) \cdot \Pr_{S\leftarrow D}[\cA(S)\in C] +\delta \xi
\end{eqnarray*}

Now we want to lower bound $ \Pr_{S'\leftarrow D'}[\cA(S')\in C]$ using eq.~(\ref{eq:dp-2}),
\begin{eqnarray*}
\Pr_{S'\leftarrow D'}[\cA(S')\in C]&=& \Pr[x\nin S'] \Pr[\cA(S')\in C ~| x\nin S']+\\
&+&\Pr[x\in S'] \Pr[\cA(S')\in C ~| x \in S']\\
&=& \left(1-\xi\right) \cdot \Pr_{S\leftarrow D}[\cA(S)\in C] +\\
&+& \xi \cdot \Pr[\cA(S')\in C ~| x \in S'] \\
&\geq& \left(1-\xi\right)\cdot \Pr_{S\leftarrow D}[\cA(S)\in C]+\\
&+& \xi\cdot e^{-\eps}\cdot (\Pr_{S\leftarrow D} [\cA(S)\in C] -\delta)\\
&=& (\xi(e^{-\eps}-1)+1) \cdot \Pr_{S\leftarrow D}[\cA(S)\in C] -e^{-\eps}\delta \xi
\end{eqnarray*}

It follows that $$\Pr_{S \leftarrow D}[\cA(S)\in C] \leq \frac{ \Pr_{S' \leftarrow D'}[\cA(S')\in C]}{(\xi(e^{-\eps}-1)+1)} + \frac{e^{-\eps} \delta \xi}{(\xi(e^{-\eps}-1)+1)} $$

We can set $$\delta' = \max\{\frac{e^{-\eps} \delta \xi}{(\xi(e^{-\eps}-1)+1)}, \delta \xi\} $$ and $$\eps ' = \ln \max\left\{ \xi(e^{\eps}-1)+1, (\xi(e^{-\eps}-1)+1)^{-1} \right\}\;.$$


\end{proof}

\subsection{Private $k$-median clustering in Metric Space}\label{subsec:metric-k-med}

Our proof is nearly identical to that of~\cite{Czumaj2007SublineartimeAA}, except that we consider continuous clusterings in metric space (see Remark in the beginning of \sectionref{priv-sla-cluster}). For ease of representation and comparison, we also adopt the notation used in~\cite{Czumaj2007SublineartimeAA}, which we recall below. 

Let $(V,d)$ be a metric space and $D \subseteq V$ be the input set, and $M$ be the diameter of $V$. Let
$${\sf med_{avg}}(D,k) = \frac{1}{\vert D \vert} \min_{\substack{C \subseteq V \\  \vert C \vert = k}} \sum_{x \in D}d(x,C) \; ,$$
denote the average cost of an optimum $k$-median clustering of $D$. Similarly, for any subset $U \subseteq D$ and $C \subseteq V$, define the average cost of a $k$-median clustering $C$ as 
$$ {\sf cost^{med}_{avg}}(U,C) =\frac{1}{\vert U \vert} \sum_{v \in U} d(v,C) \;.$$

A set of $k$ centers $C$ is a $(\rho, \varphi)$-bad solution of the $k$-median of input set $D$ if ${\sf cost^{med}_{avg}}(D,C) > \rho {\sf med_{avg}}(D,k)+\varphi$. If $C$ is not a $(\rho, \varphi)$-bad solution then it is a $(\rho, \varphi)$-good solution. 

The analysis from~\cite{Czumaj2007SublineartimeAA} involves two main steps.
\begin{enumerate}
    \item (See \lemmaref{metric-good-med-czumaj}) If $\cA(S)$ outputs clustering $C^*$, then we need to show that for a chosen sample size, with high probability , 
    $$ \alpha \cdot {\sf cost^{med}_{avg}}(S,C^*) +\gamma \leq (\alpha+\beta) {\sf med_{avg}}(D,k) + \gamma \;.$$
    \item (See \lemmaref{metric-bad-med-czumaj}) If clustering $C_b \subseteq V$ is an $(\alpha+3\beta, \gamma)$-bad solution of input set $D$, \ie, ${\sf cost^{med}_{avg}}(D,C_b) > (\alpha + \beta) {\sf med_{avg}}(D,k)+\gamma$, then, we need to show that with high probability the clustering $C_b$ is also a bad solution for the sample set $S$, 
    $$ {\sf cost^{med}_{avg}}(S,C_b) > (\alpha+\beta)  {\sf med_{avg}}(D,k) + \gamma \;.$$
\end{enumerate}
From the above two statements we get that with high probability the clustering $C^*$ (outputted by $\cA(S)$) is an $(\alpha+\beta, \gamma)$-good solution of input set $D$, in other words, $ {\sf cost^{med}_{avg}}(D,C^*) \leq (\alpha+3\beta) {\sf med_{avg}}(D,k) + \gamma$. Putting everything together, we obtain the following lemma,

\begin{lemma} \label{lem:metric-accuracy-med-czumaj} \sloppy
Let $(V,d)$ be a metric space and $D \subseteq V$. Let $0< \theta < 1$, $\alpha \geq 1$, and $\eta>0$ be approximation parameters. Assuming $\cA$ is an $(\alpha,\gamma)$-approximation algorithm for $k$-median that runs in time $T(n)$, we can draw a sample $S$ of size $s$, 
\[s \geq  c \cdot \max \left\{ \frac{M \alpha (1 +\alpha ) \ln(1/\theta)}{\eta}, \frac{M^2}{\eta^2} \cdot \left( \ln(1/\theta)+k \ln n\right) \right\} \;,\] 
where $c$ is an appropriate positive constant, and obtain a $k$-median clustering $C^*$ in time $T(s)$ such that with probability at least $1-\theta$, $$ {\sf cost^{med}_{avg}}(D,C^*) \leq \alpha{\sf med_{avg}}(D,k)+\gamma+\eta$$
\end{lemma}

\begin{lemma}\label{lem:metric-good-med-czumaj}
Let $S$ be a set of size $s$ chosen from $D \subseteq V$ i.u.r. For 
$$s \geq  \frac{3 M \alpha (\beta+\alpha) \ln(1/\theta)}{\beta^2 {\sf med_{avg}}(D,k)} \;. $$   
If an $(\alpha, \gamma)$-approximation algorithm for $k$-median $\cA$ is run on input $S$, then the following holds for the solution $C^*$ returned by $\cA$: 
$$ \Pr[{\sf cost^{med}_{avg}}(S,C^*) \leq (\alpha+\beta)  {\sf med_{avg}}(D,k) +\gamma ] \geq 1-\theta \;.$$
\end{lemma}




\begin{proof}
Let $C_{OPT}$ denote an optimal $k$-median solution for input set $D$. 
For $1 \leq i \leq s$, define random variables $X_i$ as the distance of the $i$-th point in $S$ to the nearest center of $C_{OPT}$. Then ${\sf cost^{med}_{avg}}(S,C_{OPT})=\frac{1}{s} \sum_{1 \leq i \leq s}X_i$. 
Observe that, $\E[X_i] = \sum_{x \in D} \Pr[x \text{ is sampled u.a.r. from } D] \cdot d(x, C_{OPT}) = \frac{1}{\vert D \vert } \sum_{x \in D} d(x, C_{OPT})= {\sf med_{avg}}(D,k)\;,$ also ${\sf med_{avg}}(D,k) = \frac{1}{s} \E[\sum_{1 \leq i \leq s} X_i]$. 

\begin{align*}
    \Pr\left[{\sf cost^{med}_{avg}}(S,C_{OPT}) > \left(1+\frac{\beta}{\alpha}\right) {\sf med_{avg}}(D,k) \right] &= \Pr\left[ \sum_{1 \leq i \leq s} X_i > \left(1+\frac{\beta}{\alpha}  \right) \E[\sum_{1 \leq i \leq s} X_i]\right]
\end{align*}
Each $0 \leq X_i \leq M$. Thus we can apply a Hoeffding bound, 
\begin{align*}
    &\Pr\left[ \sum_{1 \leq i \leq s} X_i > \left(1+\frac{\beta}{\alpha}  \right) \E[\sum_{1 \leq i \leq s} X_i]\right]
    \leq \exp \left( -\frac{s}{3M} \cdot {\sf med_{avg}}(D,k) \min\{(\beta/\alpha),(\beta/\alpha)^2\} \right)\\
\end{align*}
Choosing $s$ as in the lemma statement, the probability above is bounded by $\theta$. 
Since $\cA$ is an $(\alpha,\gamma)$-approximation, the lemma statement follows.
\end{proof}

Next, we need to show that any clustering $C_b$ that is a $(\alpha+3\beta,\gamma)$-bad solution of $k$-median of $D$ satisfies ${\sf cost^{med}_{avg}}(S,C_b) > (\alpha+\beta)  {\sf med_{avg}}(D,k) +\gamma$ with high probability. 

\begin{lemma}\label{lem:metric-bad-med-czumaj}
Let $S$ be a set of $s$ points chosen i.u.r. from $D \subseteq V$ such that 
$$ s \geq \frac{M^2}{2\beta^2 ({\sf med_{avg}}(D,k))^2} \cdot (\ln(1/\theta)+k \ln n )$$
Let $\bbC$ be the set of $(\alpha+3\beta,\gamma)$-bad solutions of a $k$-median clustering of $D$. Then 
$$ \Pr[\exists\ C_b \in \bbC : {\sf cost^{med}_{avg}}(S,C_b) \leq (\alpha+\beta) {\sf med_{avg}}(D,k) +\gamma ] \leq \theta$$
\end{lemma}

\begin{proof}
Consider an arbitrary $C_b\in \bbC$, and define $X_i$ as the distance of the $i$th point in $S$ from the nearest center in $C_b$. Since $C_b$ is a $(\alpha+3\beta,\gamma)$-bad solutions of a $k$-median of $D$, by definition, 
\begin{align}\label{eq:cost-med-metric-bad}
    {\sf cost^{med}_{avg}}(D,C_b) &> (\alpha + 3\beta) {\sf med_{avg}}(D,k)+\gamma
\end{align}
Now for $1 \leq i \leq s$, we have that $\E[X_i] = \frac{1}{\vert D \vert} \sum_{x \in D} d(x,C_b) = {\sf cost^{med}_{avg}}(D,C_b) $, thus 
\begin{align}\label{eq:exi-metric-med}
    \E[X_i] &> (\alpha + 3\beta) {\sf med_{avg}}(D,k)+\gamma
\end{align}
Also, 
\begin{align}\label{eq:sumxi-metric-med}
\sum_{1 \leq i \leq s} X_i = \sum_{x \in S} d(x,C_b) = s \cdot {\sf cost^{med}_{avg}}(S,C_b)\;,     
\end{align}

and $\E[\sum_{1 \leq i \leq s} X_i] = s \E[X_i]$ for any $i$.

We want to show that for any fixed $C_b \in \bbC$, $\Pr[ {\sf cost^{med}_{avg}}(S,C_b) \leq (\alpha+\beta) {\sf med_{avg}}(D,k) +\gamma]$ is low, and then take a union bound over the entire space of $\bbC$. 
\begin{align*}
    &\Pr[ {\sf cost^{med}_{avg}}(S,C_b) \leq (\alpha+\beta) {\sf med_{avg}}(D,k) +\gamma] \\
    &\text{Substituting Relation \ref{eq:sumxi-metric-med} on LHS and Relation \ref{eq:exi-metric-med} on RHS,}\\
    &=\Pr\left[ \frac{1}{s} \cdot \sum_{1 \leq i \leq s} X_i \leq \frac{(\alpha+\beta)}{(\alpha+3\beta)} \E[X_i] +\gamma\cdot \left(1-\frac{(\alpha+\beta)}{(\alpha+3\beta)}\right)\right] \\
    &=\Pr\left[\sum_{1 \leq i \leq s} X_i \leq \frac{(\alpha+\beta)}{(\alpha+3\beta)}\cdot s \cdot \E[X_i] + \frac{2s\gamma \beta}{\alpha+3\beta}\right] \\
    &=\Pr\left[\sum_{1 \leq i \leq s} X_i \leq \frac{(\alpha+\beta)}{(\alpha+3\beta)}\cdot \E[\sum_{1 \leq i \leq s} X_i] +\frac{2s\gamma \beta}{\alpha+3\beta}\right] \\
    &=\Pr\left[\sum_{1 \leq i \leq s} X_i \leq \left(\frac{(\alpha+\beta)}{(\alpha+3\beta)}  +\frac{2s\gamma \beta}{(\alpha+3\beta) \E[\sum_{1 \leq i \leq s} X_i] } \right)\cdot \E[\sum_{1 \leq i \leq s} X_i]\right] \\
    &=\Pr\left[\sum_{1 \leq i \leq s} X_i \leq \left(1- \left(\frac{2\beta}{(\alpha+3\beta)} -\frac{2s\gamma \beta}{(\alpha+3\beta) \cdot s \cdot {\sf cost_{avg}}(D,C_b) } \right)\right)\cdot \E[\sum_{1 \leq i \leq s}X_i]\right]\\    
\end{align*}
Since $0 \leq X_i \leq M$, we can apply a Hoeffding bound, 
\begin{align*}
    &\Pr[ {\sf cost^{med}_{avg}}(S,C_b) \leq (\alpha+\beta) {\sf med_{avg}}(D,k) +\gamma] \\
    &\leq \exp\left( - \frac{\E [\sum_{1 \leq i \leq s} X_i]}{2M} \cdot \left(  \frac{2\beta}{(\alpha+3\beta)} -\frac{2\gamma \beta}{(\alpha+3\beta) \cdot {\sf cost^{med}_{avg}}(D,C_b) } \right)^2 \right)\\
    &= \exp\left( - \frac{ s \cdot {\sf cost^{med}_{avg}}(D,C_b) }{2M} \cdot \left(  \frac{2\beta \cdot  {\sf cost^{med}_{avg}}(D,C_b) - 2\gamma \beta}{(\alpha+3\beta) \cdot {\sf cost^{med}_{avg}}(D,C_b) } \right)^2 \right)\\
    &= \exp\left( - \frac{2 s \beta^2 }{M \cdot (\alpha+3\beta)^2} \cdot \frac{ ({\sf cost^{med}_{avg}}(D,C_b) - \gamma)^2}{ {\sf cost^{med}_{avg}}(D,C_b) }  \right)\\
    &\leq \exp\left( - \frac{2 s \beta^2 }{M \cdot (\alpha+3\beta)^2} \cdot \frac{ ({\sf med_{avg}}(D,k))^2(\alpha+3\beta)^2}{ {\sf cost^{med}_{avg}}(D,C_b) }  \right),\ &\text{Applying relation \ref{eq:cost-med-metric-bad}}\\
\end{align*}
Now, ${\sf cost^{med}_{avg}}(D,C_b) \leq M$, therefore 
\begin{align*}
    \Pr[ {\sf cost^{med}_{avg}}(S,C_b) \leq (\alpha+\beta) {\sf med_{avg}}(D,k) +\gamma]  & \leq \exp\left( - \frac{ 2s \beta^2 }{M^2} \cdot { ({\sf med_{avg}}(D,k))^2}  \right)
\end{align*}
By union bound and using the fact that $\vert \bbC \vert \leq \binom{n}{k} \leq n^k $, 
\begin{align*}
    &\Pr[\exists\ C_b \in \bbC : {\sf cost^{med}_{avg}}(S,C_b) \leq (\alpha+\beta) {\sf med_{avg}}(D,k) +\gamma ] \\
    &\leq n^k  \cdot \exp\left( - \frac{ 2s \beta^2 }{M^2} \cdot { ({\sf med_{avg}}(D,k))^2}  \right)
\end{align*}
We choose
$$ s \geq \frac{M^2}{2\beta^2 ({\sf med_{avg}}(D,k))^2} \cdot (\ln(1/\theta)+k \ln n )$$
\end{proof}

The proof of \lemmaref{metric-accuracy-med-czumaj} is presented below. 

\begin{proof}
Let $\beta^*$ be a positive parameter that will be fixed later. Recall from \lemmaref{metric-good-med-czumaj}, the sample complexity is as follows, 
\begin{align}
    s \geq  \frac{3 M \alpha (\beta+\alpha) \ln(1/\theta)}{\beta^2 {\sf med_{avg}}(D,k)}  \;,
\end{align}
And from \lemmaref{metric-bad-med-czumaj} we have, 
\begin{align}
    s \geq \frac{M^2}{2{\beta}^2 ({\sf med_{avg}}(D,k))^2} \cdot (\ln(1/\theta)+k \ln n ) \;,
\end{align}
 Let $s$ be chosen such that sample complexity prerequisites of both \lemmaref{metric-good-med-czumaj} and \lemmaref{metric-bad-med-czumaj} hold with $\beta$ replaced with $\beta^*$.
\begin{align}
     s \geq \max \left\{ \frac{3 M \alpha (\beta^*+\alpha) \ln(1/\theta)}{{\beta^*}^2 {\sf med_{avg}}(D,k)},  \frac{M^2}{{2\beta^*}^2 ({\sf med_{avg}}(D,k))^2} \cdot \left( \ln(1/\theta)+k \ln n \right) \right\}
\end{align}
For the chosen sample complexity, we have from \lemmaref{metric-bad-med-czumaj} that with probability at least $1-\theta$, \emph{no} clustering $C\subseteq V$ that is a $(\alpha+3\beta^*,\gamma)$-bad solution of a $k$-median of $D$ satisfies the inequality ${\sf cost^{med}_{avg}}(S,C) \leq (\alpha+\beta^*) {\sf med_{avg}}(D,k) +\gamma$.

On the other hand, if we run algorithm $\cA(S)$, then by \lemmaref{metric-good-med-czumaj}, the resulting clustering $C^*$ with probability at least $1-\theta$ satisfies, ${\sf cost^{med}_{avg}}(S,C^*) \leq (\alpha+\beta^*)  {\sf med_{avg}}(D,k) +\gamma $.

Thus with probability at least $1-2\theta$, the clustering $C^*$ must be a $(\alpha+3\beta^*,\gamma)$-good solution of a $k$-median of $D$, in other words, 
\begin{align}
    \Pr[   {\sf cost^{med}_{avg}}(D,C^*) \leq (\alpha+3\beta^*) {\sf med_{avg}}(D,k) +\gamma] \geq 1-2\theta \;.
\end{align}
To complete the proof, we must remove the dependency on ${\sf med_{avg}}(D,k)$ in the sample complexity. 


\begin{itemize}
    \item Case 1: ${\sf med_{avg}}(D,k) < \eta.$ Choose $\beta^*:=(\eta/3)/ {\sf med_{avg}}(D,k)$, therefore $\beta^*\geq 1/3$ and $\beta=1/(3\beta^*)< 1$ then we get that if sample complexity $$s \geq  c \cdot \max \left\{ \frac{M \alpha (1 +\alpha \beta) \ln(1/\theta)}{\eta},  \frac{M^2}{\eta^2} \cdot \left( \ln(1/\theta)+k \ln n\right) \right\} \;,$$
    where $c$ is a certain positive constant, then with probability $1-2\theta$, $${\sf cost^{med}_{avg}}(D,C^*) \leq   (\alpha+3\beta^*) {\sf med_{avg}}(D,k) +\gamma \leq \alpha \cdot  {\sf med_{avg}}(D,k) + \gamma +\eta\;,$$ and
    \item Case 2: ${\sf med_{avg}}(D,k) \geq \eta.$ Choose $\beta^*=\eta/(3 \cdot {\sf med_{avg}}(D,k)) \leq 1/3$. Then, we get that if sample complexity $$s \geq  c \cdot \max \left\{ \frac{ M \alpha (1+\alpha) \ln(1/\theta)}{\eta},  \frac{M^2}{\eta^2} \cdot \left( \ln(1/\delta)+k \ln n\right) \right\} \;,$$
    where $c$ is a certain positive constant, then with probability $1-2\theta$, $${\sf cost^{med}_{avg}}(D,C^*) \leq (\alpha+3\beta^*) {\sf med_{avg}}(D,k) +\gamma \leq \alpha \cdot  {\sf med_{avg}}(D,k) + \gamma  + \eta\;,$$ 
\end{itemize}
\end{proof}

Given a metric space $(V,d)$ of $n$ points with diameter $M$, a private set $D \subseteq V$, Gupta \etal~\cite{GLMRT2010} modify a non-private local clustering algorithm~\cite{AGKMMP01} for solving $k$-median to make it differentially-private. Their algorithm starts off with an arbitrary set of $k$-centers and in each iteration, it swaps out an existing center in the set with a better center using the exponential mechanism, and after a sufficient number of steps, the algorithm chooses a good solution from amongst the ones seen so far. They obtain the following accuracy guarantee for their private algorithm.

\begin{theorem} \label{thm:glmrt2010}\cite{GLMRT2010}
Given a metric space $(V,d)$ of $n$ points with diameter $M$, a set $D \subseteq V$, there exists a $\eps$-differentially private $k$-median algorithm that except with probability $O(1/poly(n))$ outputs a $(6,O(M k^2 \log^2(n/\eps) ))$-approximation of a $k$-median clustering of $D$. 
\end{theorem}


We will use the algorithm in \cite{GLMRT2010} as our black-box algorithm $\cA$. By plugging in the approximation guarantees for $\cA$ into our \lemmaref{metric-accuracy-med-czumaj}, we get the following accuracy guarantee for our algorithm $\cA'$.

\begin{theorem}[Accuracy of $\cA'$]
\label{thm:metric-acc-k-med-sub}
Let $\eta>0,\ 0<\theta<1$ be approximation parameters. For an arbitrary metric space $(V,d)$ of $n$ points with diameter $M$, and a private set of points $D \subseteq V$, given the $\eps$-DP $(6,O(M k^2 \log^2(n/\eps) ))$-approximation $k$-median algorithm (from~\cite{GLMRT2010}), 
we have a $\eps'$-DP algorithm $\cA'$ (as defined in \theoremref{privacy-approx-k-med}) that can draw a sample $S \subseteq D$ of size $s$, 
$$ s =c \cdot \max \Big\{ \frac{M \ln (1/\theta)}{\eta}, \left( \frac{M}{\eta} \right)^2 \left( k \ln n + \ln \frac{1}{\theta} \right)\Big\}$$ where $c$ is an appropriate constant, and obtain a $k$-median clustering $C^*$ such that with probability at least $1-\theta$,
$${\sf cost^{med}_{avg}}(D,C^*) \leq 6{\sf med_{avg}}(D,k)+O(M k^2 \log^2(n/\eps))+\eta \;.$$
\end{theorem}

\subsection{Private $k$-median clustering in Euclidean Space}\label{subsec:euclid-k-med}
In this setting, we consider input set $D \subseteq \bbR^d$ with diameter $M$, and $\vert D \vert = n$. We use the same notation as introduced in \subsectionref{metric-k-med}, keeping in mind that now both the input set $D$ and clusterings $C$ are subsets of $\bbR^d$. The techniques are very similar to the metric space setting and we only highlight the major differences in the sequel. We first present the main lemma of this section. 
\begin{lemma} \label{lem:euclid-accuracy-med-czumaj} \sloppy
Let $D \subseteq \bbR^d$ with diameter $M$. Let $0< \theta < 1$, $\alpha \geq 1$, and $\eta>0$ be approximation parameters. Assuming $\cA$ is an $(\alpha,\gamma)$-approximation algorithm for $k$-median that runs in time $T(n)$, we can draw a sample $S$ of size $s$, 
\[s \geq  c \cdot \max \left\{ \frac{M \alpha (1 +\alpha ) \ln(1/\theta)}{\eta}, \frac{M^2}{\eta^2} \cdot \left(\ln(1/\delta)+kd \ln \left(\frac{\sqrt{d}M}{2 \eta} \right)\right) \right\} \;,\] 
where $c$ is an appropriate positive constant, and obtain a $k$-median clustering $C^*$ in time $T(s)$ such that with probability at least $1-\theta$, $$ {\sf cost^{med}_{avg}}(D,C^*) \leq \alpha{\sf med_{avg}}(D,k)+\gamma+\eta$$
\end{lemma}

Our strategy for proving \lemmaref{euclid-accuracy-med-czumaj} is identical to the strategy used in the metric setting. In fact the following statement and proof is identical to \lemmaref{metric-good-med-czumaj} which shows that the clustering outputted by the $\cA$ is also a ``good'' solution for the entire input set $D$.

\begin{lemma}\label{lem:euclid-good-med-czumaj}
Let $S$ be a set of size $s$ chosen from $D \subseteq \bbR^d$ i.u.r. For 
$$s \geq  \frac{3 M \alpha (\beta+\alpha) \ln(1/\theta)}{\beta^2 {\sf med_{avg}}(D,k)} \;. $$   
If an $(\alpha, \gamma)$-approximation algorithm for $k$-median $\cA$ is run on input $S$, then the following holds for the solution $C^*$ returned by $\cA$: 
$$ \Pr[{\sf cost^{med}_{avg}}(S,C^*) \leq (\alpha+\beta)  {\sf med_{avg}}(D,k) +\gamma ] \geq 1-\theta \;.$$
\end{lemma}

Next, we need to show that any clustering $C_b$ that is an $(\alpha+3\beta,\gamma)$-bad solution of $k$-median of $D$ satisfies with high probability ${\sf cost^{med}_{avg}}(S,C_b) > (\alpha+2\beta)  {\sf med_{avg}}(D,k) +\gamma  $. The proof below is identical to \lemmaref{metric-bad-med-czumaj}, except we use $\eta$-nets to approximate the size of the set of bad solutions denoted by $\bbC$ (see~\cite{MOP11,Czumaj2007SublineartimeAA}).

\begin{lemma}\label{lem:bad-med-euclid-czumaj}
Let $S$ be a set of $s$ points chosen i.u.r. from $D \subseteq \bbR^d$ such that 
$$ s \geq  \frac{M^2}{2\beta^2 ({\sf med_{avg}}(D,k))^2} \cdot \left(\ln(1/\delta)+kd \ln \left(\frac{\sqrt{d}M}{2 \eta} \right)\right) $$
Let $\bbC$ be the set of $(\alpha+3\beta,\gamma)$-bad solutions of a $k$-median clustering of $D$. Then 
$$ \Pr[\exists\ C_b \in \bbC : {\sf cost^{med}_{avg}}(S,C_b) \leq (\alpha+\beta) {\sf med_{avg}}(D,k) +\gamma ] \leq \delta$$
\end{lemma}

\begin{proof}
Consider an arbitrary $C_b\in \bbC$, and define $X_i$ as the distance of the $i$th point in $S$ from the nearest center in $C_b$. Since $C_b$ is a $(\alpha+3\beta,\gamma)$-bad solutions of a $k$-median of $D$, by definition, 
\begin{align}\label{eq:cost-med-euclid-bad}
    {\sf cost^{med}_{avg}}(D,C_b) &> (\alpha + 3\beta) {\sf med_{avg}}(D,k)+\gamma
\end{align}
Now for $1 \leq i \leq s$, we have that $\E[X_i] = \frac{1}{\vert D \vert} \sum_{x \in D} d(x,C_b) = {\sf cost^{med}_{avg}}(D,C_b) $, thus 
\begin{align}\label{eq:exi-med-euclid}
    \E[X_i] &> (\alpha + 3\beta) {\sf med_{avg}}(D,k)+\gamma
\end{align}
Also, 
\begin{align}\label{eq:sumxi-med-euclid}
\sum_{1 \leq i \leq s} X_i = \sum_{x \in S} d(x,C_b) = s \cdot {\sf cost^{med}_{avg}}(S,C_b)\;,     
\end{align}

and $\E[\sum_{1 \leq i \leq s} X_i] = s \E[X_i]$ for any $i$, recall that $\E[X_i]={\sf cost^{med}_{avg}}(D,C_b)$ and hence independent of $i$. 

We want to show that for any $C_b \in \bbC$, $\Pr[ {\sf cost^{med}_{avg}}(S,C_b) \leq (\alpha+2\beta) {\sf med_{avg}}(D,k) +\gamma]$ is low, and then take a union bound over the entire space of $\bbC$. 
\begin{align*}
    &\Pr[ {\sf cost^{mean}_{avg}}(S,C_b) \leq (\alpha+\beta) {\sf mean_{avg}}(D,k) +\gamma] \\
    &\text{Substituting Relation \ref{eq:sumxi-med-euclid} on LHS and Relation \ref{eq:exi-med-euclid} on RHS,}\\
      &=\Pr\left[ \frac{1}{s} \cdot \sum_{1 \leq i \leq s} X_i \leq \frac{(\alpha+\beta)}{(\alpha+3\beta)} \E[X_i] +\gamma\cdot \left(1-\frac{(\alpha+\beta)}{(\alpha+3\beta)}\right)\right] \\
    &=\Pr\left[\sum_{1 \leq i \leq s} X_i \leq \frac{(\alpha+\beta)}{(\alpha+3\beta)}\cdot s \cdot \E[X_i] + \frac{2s\gamma \beta}{\alpha+3\beta}\right] \\
    &=\Pr\left[\sum_{1 \leq i \leq s} X_i \leq \frac{(\alpha+\beta)}{(\alpha+3\beta)}\cdot \E[\sum_{1 \leq i \leq s} X_i] +\frac{2s\gamma \beta}{\alpha+3\beta}\right] \\
    &=\Pr\left[\sum_{1 \leq i \leq s} X_i \leq \left(\frac{(\alpha+\beta)}{(\alpha+3\beta)}  +\frac{2s\gamma \beta}{(\alpha+3\beta) \E[\sum_{1 \leq i \leq s} X_i] } \right)\cdot \E[\sum_{1 \leq i \leq s} X_i]\right] \\
    &=\Pr\left[\sum_{1 \leq i \leq s} X_i \leq \left(1- \left(\frac{2\beta}{(\alpha+3\beta)} -\frac{2s\gamma \beta}{(\alpha+3\beta) \cdot s \cdot {\sf cost_{avg}}(D,C_b) } \right)\right)\cdot \E[\sum_{1 \leq i \leq s}X_i]\right]\\    
\end{align*}
Since $0 \leq X_i \leq M$, we can apply a Hoeffding bound, 
\begin{align*}
    &\Pr[ {\sf cost^{med}_{avg}}(S,C_b) \leq (\alpha+\beta) {\sf med_{avg}}(D,k) +\gamma] \\
    &\leq \exp\left( - \frac{\E [\sum_{1 \leq i \leq s} X_i]}{2M} \cdot \left(  \frac{2\beta}{(\alpha+3\beta)} -\frac{2\gamma \beta}{(\alpha+3\beta) \cdot {\sf cost^{med}_{avg}}(D,C_b) } \right)^2 \right)\\
    &= \exp\left( - \frac{ s \cdot {\sf cost^{med}_{avg}}(D,C_b) }{2M} \cdot \left(  \frac{2\beta \cdot  {\sf cost^{med}_{avg}}(D,C_b) - 2\gamma \beta}{(\alpha+3\beta) \cdot {\sf cost^{med}_{avg}}(D,C_b) } \right)^2 \right)\\
    &= \exp\left( - \frac{2 s \beta^2 }{M \cdot (\alpha+3\beta)^2} \cdot \frac{ ({\sf cost^{med}_{avg}}(D,C_b) - \gamma)^2}{ {\sf cost^{med}_{avg}}(D,C_b) }  \right)\\
    &\leq \exp\left( - \frac{2 s \beta^2 }{M \cdot (\alpha+3\beta)^2} \cdot \frac{ ({\sf med_{avg}}(D,k))^2(\alpha+3\beta)^2}{ {\sf cost^{med}_{avg}}(D,C_b) }  \right),\ &\text{Applying relation \ref{eq:cost-med-euclid-bad}}\\
\end{align*}
Now, ${\sf cost^{med}_{avg}}(D,C_b) \leq M$, therefore 
\begin{align*}
    \Pr[ {\sf cost^{med}_{avg}}(S,C_b) \leq (\alpha+\beta) {\sf med_{avg}}(D,k) +\gamma]  & \leq \exp\left( - \frac{ 2s \beta^2 }{M^2} \cdot { ({\sf med_{avg}}(D,k))^2}  \right)
\end{align*}
By union bound and using the fact that $\vert \bbC \vert \leq \left(\frac{\sqrt{d} M}{2 \eta}\right)^{kd}$, 
\begin{align*}
    &\Pr[\exists\ C_b \in \bbC : {\sf cost^{mean}_{avg}}(S,C_b) \leq (\alpha+\beta) {\sf mean_{avg}}(D,k) +\gamma ] \\
    &\leq \left(\frac{\sqrt{d} M}{2 \eta}\right)^{kd}  \cdot\exp\left( - \frac{ 2s \beta^2 }{M^2} \cdot { ({\sf med_{avg}}(D,k))^2}  \right)
\end{align*}
We choose
$$ s \geq \frac{M^2}{2\beta^2 ({\sf med_{avg}}(D,k))^2} \cdot \left(\ln(1/\delta)+kd \ln \left(\frac{\sqrt{d}M}{2 \eta} \right)\right) $$
\end{proof}
 
We use identical arguments as in the proof of \lemmaref{metric-accuracy-med-czumaj} for the proof of \lemmaref{euclid-accuracy-med-czumaj}, i.e., we condition on ${\sf med_{avg}}(D,k)<\eta$ and ${\sf med_{avg}}(D,k) \geq \eta$ to remove the dependency of ${\sf med_{avg}}(D,k)$ in the sample complexity.

We now state the DP clustering results that we combine with \lemmaref{euclid-accuracy-med-czumaj} to obtain our differentially-private sublinear time $k$-median result in Euclidean space as a corollary. Given any $w$-approximation algorithm for $k$-median (respectively $k$-means), Ghazi~\etal~\cite{Ghazi2020DifferentiallyPC} use differentially-private coresets to give pure and approximate differentially-private algorithms that run in polynomial time and achieve approximation guarantees very close to that of the original algorithm.

\begin{theorem}\label{thm:ghazi2020}~\cite{Ghazi2020DifferentiallyPC}
Assume there is a polynomial-time (not necessarily DP) algorithm for $k$-median (respectively $k$-means) in $\bbR^d$ with approximation ratio $w$. Then there is an $\eps$-DP algorithm that runs in time $k^{O_\alpha(1)}\text{poly}(nd)$ and with probability $0.99$, produces a $\left(w(1+\alpha), O_{w,\alpha} \left( \left(\frac{kd+k^{O_\alpha(1)}}{\eps} \right)  \text{poly log }n \right) \right)$- approximation for $k$-median (respectively $k$-means). 

Moreover, there is an $(\eps,\delta)$-DP algorithm with the same runtime and approximation ratio but with additive error $O_{w,\alpha} \left(\left( \frac{k\sqrt{d}}{\eps} \cdot \text{poly log}\left(\frac{k}{\delta}\right) \right) + \left(\frac{k^{O_\alpha(1)}}{\eps} \cdot \text{poly log }n \right)  \right)$.
\end{theorem}

Note that the state-of-the-art non-private algorithm for $k$-median achieves an approximation ratio of $w=2.633$~\cite{ahmadian2019better}. We use the algorithm from~\cite{Ghazi2020DifferentiallyPC} as our black-box algorithm $\cA$, and by plugging in the approximation guarantees of $\cA$ as stated in \theoremref{ghazi2020} with \lemmaref{euclid-accuracy-med-czumaj}, we obtain the following accuracy guarantees for our sampling algorithm $\cA'$ in the pure differential privacy as well as the approximate differential privacy settings.

\begin{theorem}[Accuracy of $\cA'$ for pure and approximate DP] \sloppy
\label{thm:euclid-acc-k-med-sub}
Let $\eta>0,\ 0<\theta<1,$ constant $\alpha$ be approximation parameters, along with approximation ratio $w$. For private set $D \subseteq \bbR^d$ with diameter $M$, and an $\eps$-DP $\Bigl(w(1+\alpha), O \Bigl( \Bigl(\frac{kd+k^{O(1)}}{\eps} \Bigr)  \text{poly log }n \Bigr) \Bigr)$-approximation $k$-median algorithm (from~\cite{Ghazi2020DifferentiallyPC}), that runs in time $k^{O(1)}\text{poly}(nd)$, we have a $\eps'$-DP algorithm $\cA'$ that can draw a sample $S \subseteq D$ of size $s$, 
\[ s = c \cdot \max \left\{ \frac{M w(1+\alpha)(1+w(1+\alpha)) \ln(1\theta)}{\eta}, \frac{M^2}{\eta^2} \left( \ln(1/\theta) + kd \ln\left( \frac{\sqrt{d} M}{2 \eta}\right) \right)\right\}\] 
where $c$ is an appropriate constant, and obtain a $k$-median clustering $C^*$ in time $k^{O(1)}\text{poly}(s d)$  such that with probability at least $1-\theta$, ${\sf cost^{med}_{avg}}(D,C^*) \leq w(1+\alpha){\sf med_{avg}}(D,k)+O \Bigl( \Bigl(\frac{kd+k^{O(1)}}{\eps} \Bigr)  \text{poly log }n \Bigr)+\eta$.

Moreover, by using the $(\eps,\delta)$-DP algorithm from~\cite{Ghazi2020DifferentiallyPC} with the same runtime and approximation ratio but with additive error $\gamma':=O \Bigl(\Bigl( \frac{k\sqrt{d}}{\eps} \cdot \text{poly log}\Bigl(\frac{k}{\delta}\Bigr) \Bigr) + \Bigl(\frac{k^{O(1)}}{\eps} \cdot \text{poly log }n \Bigr)  \Bigr)$, we obtain a $(\eps',\delta')$-DP algorithm $\cA'$ that draws a sample of the same size, and obtains a $k$-median clustering such that with probability at least $1-\theta$, ${\sf cost^{med}_{avg}}(D,C^*) \leq w(1+\alpha){\sf med_{avg}}(D,k)+\gamma'+\eta$. Privacy parameters $\eps',\delta'$ are as defined in \theoremref{privacy-approx-k-med}.
\end{theorem}
Note that the state-of-the-art non-private algorithm for $k$-median achieves an approximation ratio of $w=2.633$~\cite{ahmadian2019better}.

\subsection{Private $k$-means clustering in Metric Space}\label{subsec:k-means-metric}
We follow the techniques of Czumaj~\etal~\cite{Czumaj2007SublineartimeAA} and extend their sublinear $k$-means clustering analysis to work for black-box polynomial-time $k$-means algorithms that have an additive factor of $\gamma>0$. The analysis is almost identical to that of the $k$-median problem in metric space, except now, we work with the square of the metric distance function. We combine this extension with the existing private $k$-means clustering algorithm~\cite{balcanDLMZ17a} to obtain a private sublinear-time $k$-means clustering algorithm in metric space.

For ease of representation and comparison, we again adopt the notation used in~\cite{Czumaj2007SublineartimeAA}, which we recall below. 

Let $(V,d)$ be a metric space and $D \subseteq V$ be the input set, and $M$ be the diameter of $V$. Let
$${\sf mean_{avg}}(D,k) = \frac{1}{\vert D \vert} \min_{\substack{C \subseteq V \\  \vert C \vert = k}} \sum_{x \in D}d(x,C)^2 \; ,$$
denote the average cost of an optimum $k$-mean clustering of $D$. Similarly, for any subset $U \subseteq D$ and $C \subseteq V$, define the average cost of a $k$-mean clustering $C$ as 
$$ {\sf cost^{mean}_{avg}}(U,C) =\frac{1}{\vert U \vert} \sum_{v \in U} d(v,C)^2 \;.$$

We first state the main lemma of this section. 
\begin{lemma} \label{lem:metric-accuracy-mean-approx} \sloppy
Let $0< \theta < 1$, $\alpha \geq 1$, and $\eta>0$ be approximation parameters. For $D \subseteq \bbR^d$, assuming an $(\alpha,\gamma)$-approximation $k$-means algorithm that runs in time $T(n)$, we can draw a sample $S$ of size $s$,  
\[s \geq  c \cdot \max \left\{ \frac{M^2 \alpha (1 +\alpha ) \ln(1/\theta)}{\eta}, \frac{M^4}{\eta^2} \cdot \left( \ln(1/\theta)+k \ln n\right) \right\} \;,\] 
where $c$ is a positive constant, and obtain a $k$-means clustering $C^*$ in time $T(\vert S \vert)$ such that with probability at least $1-\theta$, ${\sf cost^{mean}_{avg} }(D,C^*) \leq \alpha {\sf mean_{avg}}(D,k) + \gamma+ \eta$.
\end{lemma}

The proofs of the following two lemmas are identical to those in \subsectionref{metric-k-med}, barring the fact that the distance function is now squared. We state the lemmas here for the sake of completeness, and note that the proof for \lemmaref{metric-accuracy-mean-approx} will follow by considering the sample complexity that satisfies both \lemmaref{metric-good-mean-czumaj} and \lemmaref{metric-bad-mean-czumaj} and then  removing the dependence of ${\sf mean_{avg}}(D,k)$ from the expression obtained. 
\begin{lemma}\label{lem:metric-good-mean-czumaj}
Let $S$ be a set of size $s$ chosen from $D \subseteq V$ i.u.r. For 
$$s \geq  \frac{3 M^2 \alpha (\beta+\alpha) \ln(1/\theta)}{2\beta^2 {\sf mean_{avg}}(D,k)} \;. $$   
If an $(\alpha, \gamma)$-approximation algorithm for $k$-means $\cA$ is run on input $S$, then the following holds for the solution $C^*$ returned by $\cA$: 
$$ \Pr[{\sf cost^{mean}_{avg}}(S,C^*) \leq (\alpha+\beta)  {\sf mean_{avg}}(D,k) +\gamma ] \geq 1-\theta \;.$$
\end{lemma}

\begin{lemma}\label{lem:metric-bad-mean-czumaj}
Let $S$ be a set of $s$ points chosen i.u.r. from $D \subseteq V$ such that 
$$ s \geq  \frac{M^4}{2\beta^2 ({\sf mean_{avg}}(D,k))^2} \cdot (\ln(1/\theta)+k \ln n )$$
Let $\bbC$ be the set of $(\alpha+3\beta,\gamma)$-bad solutions of a $k$-means clustering of $D$. Then 
$$ \Pr[\exists\ C_b \in \bbC : {\sf cost^{mean}_{avg}}(S,C_b) \leq (\alpha+\beta) {\sf mean_{avg}}(D,k) +\gamma ] \leq \theta$$
\end{lemma}

Following the techniques of~\cite{GLMRT2010}, \cite{balcanDLMZ17a} extended their results to the private $k$-means setting by adapting their analysis to the non-private local search approximation algorithm for $k$-means clustering~\cite{kanungo2004local}. 

\begin{theorem} \label{balc2010}\cite{balcanDLMZ17a}
Given a metric space $(V,d)$ of $n$ points with diameter $M$, a set $D \subseteq V$, there exists an $\epsilon$-differentially private $k$-means algorithm that with probability at least 0.99 produces a $\bigl( 30,O\left((M^2 k^4/\epsilon)\cdot \log^2 n \right) \bigr)$-approximation for $k$-means clustering. 
\end{theorem}

We use the algorithm from~\cite{balcanDLMZ17a} as our private black-box $k$-means clustering algorithm $\cA$. By plugging in the approximation guarantees of $\cA$ into our \lemmaref{metric-accuracy-mean-approx}, we obtain the following accuracy guarantee for our sublinear sampling algorithm $\cA'$. 
\begin{theorem}[Accuracy of $\cA'$]
\label{thm:metric-acc-k-mean-sub}
Let $(V,d)$ be a metric space of $n$ points with diameter $M$. Let $0< \theta < 1$, and $\eta>0$ be approximation parameters. For private set $D \subseteq V$,  and an $\epsilon$-DP $\left( 30,O\left((M^2 k^4/\epsilon)\cdot \log^2 n \right) \right)$-approximation $k$-means algorithm (from~\cite{balcanDLMZ17a}), that runs in time $T(n)$, we have a $\epsilon'$-DP algorithm $\cA'$ (as defined in \theoremref{privacy-approx-k-med}) that can draw a sample $S \subseteq D$ of size $s$, 
$$  s \geq  c \cdot \max \Bigl\{ \frac{ M^2 \ln(1/{\theta})}{\eta}, \frac{M^4}{\eta^2} \cdot \left( \ln(1/\theta)+k \ln n \right) \Bigr\} \;,  $$
where $c$ is a positive constant, and obtain a $k$-means clustering $C^*$ in time $T(s)$  such that with probability at least $1-\theta$, ${\sf cost^{mean}_{avg}}(S,C^*) \leq 30 {\sf mean_{avg}}(D,k) +O\left((M^2 k^4/\epsilon)\cdot \log^2 n \right)  + \eta$.
\end{theorem}

\subsection{Private $k$-means clustering in Euclidean Space }
The extension of the $k$-means analysis to Euclidean space involves the same steps as outlined in \subsectionref{k-means-metric}, but similar to the analysis for $k$-median in Euclidean space, we need to consider $\eta$-nets to estimate the size of possible clusterings (see~\subsectionref{euclid-k-med}). The main lemma is presented below. 
\begin{lemma} \label{lem:euclid-accuracy-mean-approx}
For $D \subseteq \bbR^d$, assuming an $(\alpha,\gamma)$-approximation $k$-means algorithm that runs in time $T(n)$, we can draw a sample $S$ of size $s$,  $$s \geq  c \cdot \max \left\{ \frac{M^2 \alpha (1 +\alpha ) \ln(1/\theta)}{\eta},  \frac{M^4}{\eta^2} \cdot \left( \ln(1/\theta)+kd \ln \left(\frac{\sqrt{d}M}{2 \eta} \right)\right) \right\} \;,$$ where $c$ is a positive constant, and obtain a $k$-means clustering $C^*$ in time $T(s)$ such that with probability at least $1-\theta$, ${\sf cost^{mean}_{avg}}(S,C^*) \leq \alpha {\sf mean_{avg}}(D,k) + \gamma+ \eta$. 
\end{lemma}

Note that the state-of-the-art non-private algorithm for $k$-means achieves an approximation ratio of $w=6.358$~\cite{ahmadian2019better}. We use the private $k$-means algorithm by Ghazi~\etal~\cite{Ghazi2020DifferentiallyPC} (See~\theoremref{ghazi2020}) as our black-box private $k$-means clustering algorithm $\cA$. By plugging in the approximation guarantees for $\cA$ to \lemmaref{euclid-accuracy-mean-approx} we obtain the following accuracy guarantees for the sampling algorithm $\cA'$ in both the pure approximate differential privacy setting. 

\begin{theorem}[Accuracy of $\cA'$ for pure and approximate DP] \sloppy
\label{thm:euclid-acc-k-mean-sub}
Let $\eta>0,\ 0<\theta<1,$ constant $\alpha$ be approximation parameters, along with approximation ratio $w$. For private set $D \subseteq \bbR^d$ with diameter $M$, and an $\epsilon$-DP $\Bigl(w(1+\alpha), O \Bigl( \Bigl(\frac{kd+k^{O(1)}}{\epsilon} \Bigr)  \text{poly log }n \Bigr) \Bigr)$-approximation $k$-means algorithm (from~\cite{Ghazi2020DifferentiallyPC}), that runs in time $k^{O(1)}\text{poly}(nd)$, we have an $\eps'$-DP algorithm $\cA'$ that can draw a sample $S \subseteq D$ of size $s$, 
$$s \geq  c \cdot \max \Bigl\{ \frac{M^2 w(1+\alpha) (1 +w(1+\alpha) ) \ln(1/\theta)}{\eta},  \\ \frac{M^4}{\eta^2} \cdot \left( \ln(1/\theta)+kd \ln \left(\frac{\sqrt{d}M}{2 \eta} \right)\right) \Bigr\} \;,   $$ 
where $c$ is an appropriate constant and obtain a $k$-means clustering $C^*$ in time $k^{O(1)}\text{poly}(s d)$  such that with probability at least $1-\theta$, ${\sf cost^{mean}_{avg}}(S,C^*) \leq w(1+\alpha) {\sf mean_{avg}(D,k)} +O \Bigl( \Bigl(\frac{kd+k^{O(1)}}{\epsilon} \Bigr)  \text{poly log }n  \Bigr) + \eta$.

Moreover, by using the $(\eps,\delta)$-DP algorithm from~\cite{Ghazi2020DifferentiallyPC} with the same runtime and approximation ratio but with additive error $\gamma':=O \Bigl(\Bigl( \frac{k\sqrt{d}}{\epsilon} \cdot \text{poly log}\Bigl(\frac{k}{\delta}\Bigr) \Bigr) + \Bigl(\frac{k^{O(1)}}{\epsilon} \cdot \text{poly log}n \Bigr)  \Bigr)$, we obtain a $(\epsilon',\delta')$-DP algorithm $\cA'$ that draws a sample of the same size, and obtains a $k$-means clustering such that with probability at least $1-\theta$, ${\sf cost^{mean}_{avg}}(S,C^*) \leq w(1+\alpha){\sf mean_{avg}}(D,k)+\gamma' + \eta$. Privacy parameters $\epsilon',\delta'$ are as defined in \theoremref{privacy-approx-k-med}.
\end{theorem}

\section{Group Privacy in Sublinear setting}



In this section, we give a group privacy result that holds for \emph{any} sampling algorithm $\cA'(D)$ that samples a set $S$ from the input set $D$ by independently sampling with probability $\xi$ and runs an $\epsilon$-DP algorithm $\cA$ on $S$.
Let $D'$ be a set that differs on $g$ elements with respect to $D$, and $0 \leq T \leq g$ be a threshold. Define 
$\delta_{T,\xi,g}:=1-\sum^T_{j=0} {\binom{g}{j}} \xi^j (1-\xi)^{g-j} \;,$
in other words, $\delta_{T,\xi,g}$ is the probability of choosing more than $T$ elements that differ from elements in $D'$ in the sample $S$.  

Given that $\cA$ is $\epsilon$-DP, we have already shown that $\cA'$ is $\epsilon'$-DP (see \theoremref{privacy-approx-k-med}). In the following theorem, we show that $\cA'$ also gives us better group privacy guarantees. 
\begin{theorem}
If $\cA'$ is an $\epsilon'$-DP sampling algorithm (as described above) then it gives $(T \cdot \epsilon', \delta_{T,\xi,g})$-privacy for groups of size $g$, where $\delta_{T,\xi,g}:=1-\sum^T_{j=0} {\binom{g}{j}} \xi^j (1-\xi)^{g-j}$. 
\end{theorem}
\begin{proof}
Consider two sets $D$ and $D'$ that differ on $g$ elements, \ie, $\vert D \vert = \vert D' \vert +g$ and set $S \subseteq D$ sampled independently w.p. $\xi$. 
Define the random variable $Y$ to be the number of elements in $S$ sampled from the $g$ differing elements. Fix an output set $C$ in the output space of $\cA'$. Then 
\begin{align*}
    &\Pr[\cA'(D) \in C] \\&=\sum^g_{i=0} \Pr[\cA'(D) \in C, Y=i] \\ 
    &= \sum^g_{i=0} \Pr[\cA'(D) \in C \vert Y=i] \Pr[Y=i] \\
    &=  \sum^T_{i=0} \Pr[\cA'(D) \in C \vert Y=i] \Pr[Y=i] 
    +\sum^g_{i=T+1} \Pr[\cA'(D) \in C \vert Y=i] \Pr[Y=i]\\
    &\text{Applying the naive group privacy bound for each term }\Pr[\cA'(D) \in C \vert Y=i]\text{ in the first sum,}\\
    &\leq \sum^T_{i=0} e^{\epsilon \cdot i} \Pr[\cA'(D') \in C] \Pr[Y=i] +\sum^g_{i=T+1} \Pr[\cA'(D) \in C \vert Y=i] \Pr[Y=i]
\end{align*}
Observe that $ \sum^g_{i=T+1} \Pr[\cA'(D) \in C \vert Y=i] \Pr[Y=i] \leq  \sum^g_{i=T+1} \Pr[Y=i] \leq \delta_{T,\xi,g}$, therefore, 
\begin{align*}
    \Pr[\cA'(D) \in C] \leq e^{\epsilon \cdot T} \Pr[\cA'(D') \in C] + \delta_{T,\xi,g} \;.
\end{align*}
\end{proof}


We demonstrate how in many instances, our sampling algorithm $\cA'$ achieves better group privacy guarantees for chosen $\xi$ and $T$ such that $T\ll g$. 
    (1) If we sample each element of the input set with probability $\xi=1/\sqrt{g}$, and set threshold $T=2\sqrt{g}$, then $\cA'$ is $(2\sqrt{g}\epsilon',\delta_{T,\xi,g})$ for $\delta_{T,\xi,g}$ negligible in $g$. 
    (2) If we sample each element of the input set with probability $\xi=1/\log({g})$, and set threshold $T=2g/\log(g)$, then $\cA'$ is $((2g/\log( g))\epsilon',\delta_{T,\xi,g})$ for $\delta_{T,\xi,g}$ negligible in $g$. 


\section{Acknowledgements}
Elena would like to thank  Marek Elias, Michael Kapralov  and Aida Mousavifar for initial discussions on this topic while she was visiting EPFL. She also thanks her EPFL hosts for their hospitality.

\newpage
\bibliographystyle{plain}
\bibliography{references}
\end{document}

%% file: headers.tex
\usepackage{amssymb}
\usepackage{amsmath}
\usepackage{amsthm}
\usepackage{amsfonts}
\usepackage{mathtools}
\usepackage{ upgreek }
\usepackage{xspace}
\usepackage{bm}
\usepackage{nicefrac}
\usepackage{bbm}
\usepackage{boxedminipage}
\usepackage{xparse}
\usepackage[x11names]{xcolor}
\usepackage{float}
\usepackage{multirow}
\usepackage{graphicx}
\usepackage{subcaption}
\usepackage{ifthen}
\usepackage{url}
\usepackage{algorithm,algorithmic}
\usepackage{verbatim}
\usepackage{array}
\usepackage{wrapfig}
\usepackage{geometry}
\usepackage{enumitem}
\usepackage{titlesec}
\usepackage{pgfplots}
\usepackage{stmaryrd}
\usepackage{adjustbox}
\usepackage{soul} 
\usepackage{hyperref}
\usepackage[subtle]{savetrees}

\makeatletter
\newcommand*{\inlineequation}[2][]{%
  \begingroup
    \refstepcounter{equation}%
    \ifx\\#1\\%
    \else
      \label{#1}%
    \fi
    \relpenalty=10000 %
    \binoppenalty=10000 %
    \ensuremath{%
      #2%
    }%
    ~\@eqnnum
  \endgroup
}
\makeatother

\def\final{0}  
\def\iflong{\iffalse}
\ifnum\final=0  
\newcommand{\jnote}[1]{{\color{red}[{Jeremiah: \bf #1}]\marginpar{\color{red}*}}}
\newcommand{\enote}[1]{{\color{green}[{\small Elena: \bf #1}]\marginpar{\color{red}*}}}
\newcommand{\tanote}[1]{{\color{purple}[\small {Tamalika: \bf #1}]\marginpar{\color{red}*}}}
\else 
\newcommand{\jnote}[1]{}
\newcommand{\enote}[1]{}
\newcommand{\tanote}[1]{}
\newcommand{\todo}[1]{}
\fi

\makeatletter
\newtheorem*{rep@theorem}{\rep@title}
\newcommand{\newreptheorem}[2]{%
\newenvironment{rep#1}[1]{%
 \def\rep@title{\theoremref{##1} Restated}%
 \begin{rep@theorem}}%
 {\end{rep@theorem}}}
\makeatother

\makeatletter
\newtheorem*{rep@lemma}{\rep@title}
\newcommand{\newreplemma}[2] {%
\newenvironment{rep#1}[1]{%
 \def\rep@title{\lemmaref{##1} Restated}%
 \begin{rep@lemma}}%
 {\end{rep@lemma}}}
\makeatother

\newtheorem{theorem}{Theorem}
\newreptheorem{theorem}{Theorem}

\newtheorem{definition}{Definition}
\newtheorem{lemma}{Lemma}

\newreplemma{lemma}{Lemma}

\theoremstyle{definition}
\newtheorem*{remark}{Remark}

\newcommand{\E}[0]{\mathop{\bbE}\xspace}

\newcommand{\namedref}[2]{\hyperref[#2]{#1~\ref*{#2}}\xspace}

\newcommand{\lemmaref}[1]{\namedref{Lemma}{lem:#1}}

\newcommand{\theoremref}[1]{\namedref{Theorem}{thm:#1}}

\newcommand{\sectionref}[1]{\namedref{Section}{sec:#1}}
\newcommand{\subsectionref}[1]{\namedref{Subsection}{subsec:#1}}

\newcommand{\ie}{\text{i.e.}\xspace}

\newcommand{\etal}{\text{et al.}\xspace}

\newcommand{\nin}[0]{\ensuremath{\not\in}\xspace}

  \newcommand{\eps}[0]{\ensuremath{\varepsilon}}
  \let\epsilon\eps
  
  \newcommand{\cA}{\ensuremath{{\mathcal A}}\xspace}

  \newcommand{\cM}{\ensuremath{{\mathcal M}}\xspace}

  \newcommand{\bbC}{\ensuremath{{\mathbb C}}\xspace}
  
  \newcommand{\bbE}{\ensuremath{{\mathbb E}}\xspace}

  \newcommand{\bbR}{\ensuremath{{\mathbb R}}\xspace}

\usepackage{scalerel,stackengine}
\stackMath
\newcommand\reallywidehat[1]{%
	\savestack{\tmpbox}{\stretchto{%
			\scaleto{%
				\scalerel*[\widthof{\ensuremath{#1}}]{\kern-.5pt\bigwedge\kern-.5pt}%
				{\rule[-\textheight/2]{1.5ex}{\textheight}}
			}{\textheight}%
		}{0.8ex}}%
	\stackon[1pt]{#1}{\tmpbox}%
}